\documentclass[sigconf]{acmart}

\AtBeginDocument{%
  }

\copyrightyear{2025}
\acmYear{2025}
\setcopyright{acmlicensed}\acmConference[KDD '25]{Proceedings of the 31st ACM SIGKDD Conference on Knowledge Discovery and Data Mining V.1}{August 3--7, 2025}{Toronto, ON, Canada}
\acmBooktitle{Proceedings of the 31st ACM SIGKDD Conference on Knowledge Discovery and Data Mining V.1 (KDD '25), August 3--7, 2025, Toronto, ON, Canada}
\acmDOI{10.1145/3690624.3709211}
\acmISBN{979-8-4007-1245-6/25/08}

\usepackage{mathtools}
\usepackage{amsthm}
\newtheorem{theorem}{Theorem}

\newtheorem{assumption}{Assumption}

\newtheorem{lemma}{Lemma}

\newtheorem{objective}{Objective}
\usepackage{dsfont}

\usepackage{algorithm}
\usepackage[noend]{algorithmic}
\usepackage{xpatch}
\makeatletter
\xpatchcmd{\algorithmic}
  {\newcommand{\STATE}{\ALC@it}}
  {\newcommand{\STATE}{\@ifstar\STATEstar\STATEnostar}}
  {}{}
\newcommand{\STATEstar}{\item[]}
\newcommand{\STATEnostar}{\ALC@it}
\usepackage{multicol}
\usepackage{multirow}
\usepackage{subfigure}
\usepackage{enumitem}

\begin{document}

\title{Personalized Language Model Learning on Text Data Without User Identifiers}

\author{Yucheng Ding}
\authornote{Equal Contribution. Work done during their internships at Tencent.}
\affiliation{%
  \institution{Shanghai Jiao Tong University}
  \city{Shanghai}
  \country{China}
}
\email{yc.ding@sjtu.edu.cn}
\orcid{0000-0001-6095-4947}

\author{Yangwenjian Tan}
\authornotemark[1]
\affiliation{%
  \institution{Shanghai Jiao Tong University}
  \city{Shanghai}
  \country{China}
}
\email{liuan18@sjtu.edu.cn}

\author{Xiangyu Liu}
\affiliation{%
  \institution{WeChat AI, Tencent}
  \city{Beijing}
  \country{China}}
\email{xiangyuliu@tencent.com}

\author{Chaoyue Niu}
\authornote{Corresponding author}
\affiliation{%
  \institution{Shanghai Jiao Tong University}
  \city{Shanghai}
  \country{China}
}
\email{rvince@sjtu.edu.cn}

\author{Fandong Meng}
\affiliation{%
  \institution{WeChat AI, Tencent}
  \city{Beijing}
  \country{China}}
\email{fandongmeng@tencent.com}

\author{Jie Zhou}
\affiliation{%
  \institution{WeChat AI, Tencent}
  \city{Beijing}
  \country{China}}
\email{withtomzhou@tencent.com}

\author{Ning Liu}
\affiliation{%
  \institution{Shanghai Jiao Tong University}
  \city{Shanghai}
  \country{China}
}
\email{ningliu@sjtu.edu.cn}

\author{Fan Wu}
\affiliation{%
  \institution{Shanghai Jiao Tong University}
  \city{Shanghai}
  \country{China}
}
\email{fwu@cs.sjtu.edu.cn}

\author{Guihai Chen}
\affiliation{%
  \institution{Shanghai Jiao Tong University}
  \city{Shanghai}
  \country{China}
}
\email{gchen@cs.sjtu.edu.cn}

\renewcommand{\shortauthors}{Yucheng Ding et al.}

\begin{abstract}
In many practical natural language applications, user data are highly sensitive, requiring anonymous uploads of text data from mobile devices to the cloud without user identifiers. However, the absence of user identifiers restricts the ability of cloud-based language models to provide personalized services, which are essential for catering to diverse user needs. The trivial method of replacing an explicit user identifier with a static user embedding as model input still compromises data anonymization. In this work, we propose to let each mobile device maintain a user-specific distribution to dynamically generate user embeddings, thereby breaking the one-to-one mapping between an embedding and a specific user. We further theoretically demonstrate that to prevent the cloud from tracking users via uploaded embeddings, the local distributions of different users should either be derived from a linearly dependent space to avoid identifiability or be close to each other to prevent accurate attribution. Evaluation on both public and industrial datasets using different language models reveals a remarkable improvement in accuracy from incorporating anonymous user embeddings, while preserving real-time inference requirement. 
\end{abstract}







\maketitle


\section{Introduction}
 
To provide intelligent services for a large and diverse population of mobile device users, deep language models, particularly Transformer-based networks, have been widely deployed in various practical natural language applications, such as intelligent keyboards~\cite{intelassoc,ChineseGPT}, which predict next words or sentences based on user input context; voice assistants~\cite{siri,cortana,google_voice}, which interpret and respond to user commands; and personal chatbots~\cite{facebook_chat, duolingo}, which engage in conversation based on user queries. The mainstream way to train a language model on the cloud involves collecting user data to build a training dataset. However, due to the sensitivity of user texts, such as messaging records and personal corpora, service providers with strict privacy requirements avoid collecting personal information, ensuring that the cloud-based dataset remains anonymized.




Data anonymization, however, makes it challenging to incorporate personalization into the cloud-based language models. Nonetheless, personalization is essential for delivering customized services to diverse users, and has become practical requirements in different industrial applications. One typical application is personalized recommender systems~\cite{DIN,PLE}. Different from the data settings considered in this work, recommender systems currently allow the cloud to collect user data along with personal information, such as user identifier (ID). The personal information is first encoded into a static user embedding and then fed to upper layers to generate personalized model outputs. Another line of work on on-device learning~\cite{MPDA} or cross-device federated learning (FL)~\cite{perFLvbi, perFLshare, perFLdisentangle} assumes that all fields of user data cannot be uploaded or exposed to the cloud, thus eliminating the need for data anonymization. These work proposed to offload a model to mobile devices for local personalized finetuning and real-time inference. However, the model deployed on the resource-constraint mobile devices must be light-weight enough to meet real-time inference requirement, which is, instead, not feasible for applications based on complex language models. Driven by the benefit of personalization and the infeasibility of existing methods in our considered language applications, we consider how to learn a personalized language model over  anonymous user data on the cloud. 

\begin{figure*}[!t]
\centering
\subfigure[On-Device Training of Embedding Distribution Parameters]{
\includegraphics[height=0.17\textwidth]{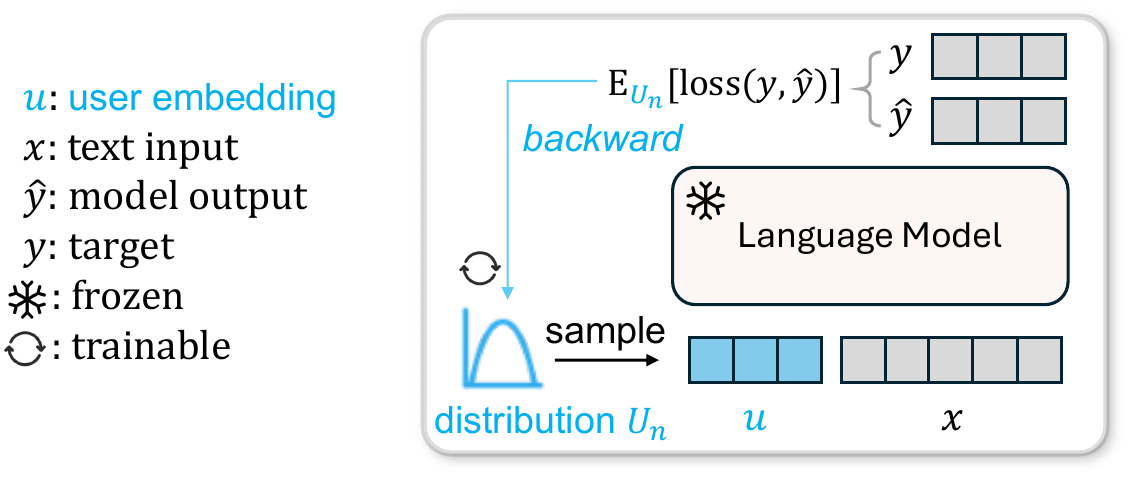}
}
\subfigure[Cloud-Based Training of Language Model]{
\includegraphics[height=0.17\textwidth]{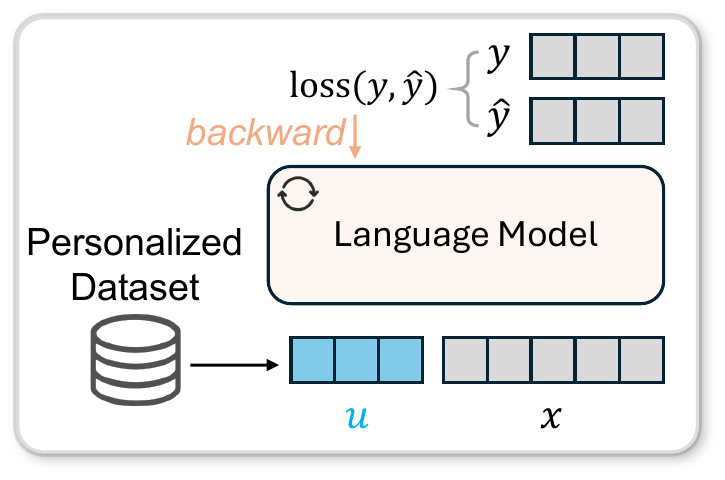}
}
\subfigure[Cloud-Based Real-Time Serving]{
\includegraphics[height=0.17\textwidth]{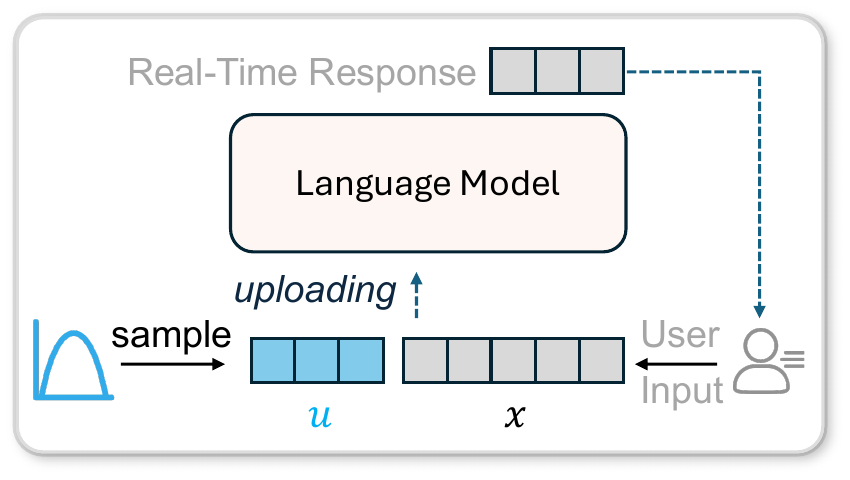}
}
\caption{Workflow of IDfree-PL. In the training phase, (a) each mobile device first trains an optimal user distribution $\mathcal{U}_n$ locally and then uploads the concatenation of user embedding sampled from $\mathcal{U}_n$ and original text to the cloud; and (b) the cloud trains the language model over the collected samples. In the inference phase, (c) the cloud provide personalized services to users based on their sampled embeddings and text inputs.}\label{workflow}
\end{figure*}

Following the common practice in recommender systems, we take user embeddings as the representation of personal information in language models. However, the trivial way of letting each mobile device upload the concatenation of a user embedding and the original text data still compromises data anonymization, due to the one-to-one mapping between a static embedding and a specific user. To deal with this problem, we propose to let each user maintain a local distribution on the mobile device and dynamically sample embeddings from this distribution. Such a new paradigm of sampling dynamic user embeddings from user-specific distributions effectively resolves the contradiction between data anonymization and personalization. For anonymization, the one-to-one mapping between users and embeddings can be turned to one-way mapping or many-to-one mapping. Specifically, in a one-way mapping, the cloud observes a mixture of embeddings, but cannot inversely identify the local distributions of different users that generated these embeddings; and in a many-to-one mapping, any embedding observed by the cloud could be generated from the local distributions of multiple users. For personalization, the parameters of the user-specific distribution for generating embeddings are optimized over the user's local data to enhance the model performance. 

%


In this work, we propose a user \underline{ID}entifier \underline{free} \underline{P}ersonalized \underline{L}earning (IDfree-PL) framework and depict the workflow in Figure~\ref{workflow}. IDfree-PL imposes requirements on the choices of user-specific distributions to ensure data anonymization and trains the parameters of the chosen distributions to achieve personalized model performance. To implement a one-way user-embedding mapping, we theoretically establish that the function space of user-specific distributions for sampling embeddings must be linearly dependent, ensuring that the mixture of distributions from multiple users is non-identifiable to the cloud. A common example of such a distribution is Beta distribution. Alternatively, to implement a many-to-one user-embedding mapping, we reveal that there is no limit on the distribution space, but user-specific distributions need to be close to each other, such that the probability of the cloud wrongly attributing any sampled embedding to its source distribution is high. Moreover, to obtain the parameter of each user's local distribution, the key difference from conventional user embedding method is the application of reparameterization technique in training the distribution parameters while freezing the language model to minimize the loss over the local user data. Given new data that concatenates the sampled user embeddings and the original texts, the cloud fine-tunes the language model for personalized input adaptation and further provides real-time inference service.

We summarize the key contributions of this work as follows:
\begin{itemize}
    \item We identify a new practical requirement in cloud-based natural language applications: how to learn a personalized language model on the cloud over the text data anonymously uploaded from mobile devices without user identifiers.
    \item We, for the first time, propose to dynamically sample user embeddings from each user's local distribution and concatenate them with original text data. Such design enables personalization by optimizing distribution parameters over local data, while achieving data anonymization by breaking the one-to-one mapping between users and embeddings.
    \item We further theoretically demonstrate the required properties of user-specific distributions for generating anonymous embeddings, which should be in a linearly dependent space to guarantee non-identifiability or be close from each other to ensure the high probability of misattribution. 
    \item We extensively evaluate\footnote{The code is available from https://github.com/sjtu-yc/IDfree-Personalized-Learning.} the proposed design IDfree-PL over three public datasets and one industrial dataset using four representative language models for four common tasks. Compared with the cloud-based model without personalization, IDfree-PL improves the inference accuracy by up to 5.69\% while adding at most 0.01s of inference latency and keeping data anonymization. 
\end{itemize}


\section{Preliminaries}\label{sec:bg}


\subsection{Application Requirement}

In the natural language applications, users' local text data are highly sensitive. Therefore, service providers for applications with stringent privacy requirements avoid collecting and storing users' personal information, making data anonymization a practical requirement. We take the anonymous mobile keyboard (such as WeChat IME) as an example. The service provider establish communication channels for anonymous data uploading, the APP users anonymously upload their local data, and the cloud stores the mixture of numerous users' data without personal user information, such as user ID or IP address.


However, incorporating personal user information into the cloud-based dataset can enhance the language model's ability to deliver customized services to diverse users, thus significantly improving its inference performance. We experimentally validate the improvement of introducing user ID to a typical sentiment analysis task. We take the Amazon-Kindle dataset, which consists of reviews on the electronic books and the corresponding ratings from 1,435 real-world users after preprocessing. We use GPT2~\cite{gpt2}, T5~\cite{t5}, or Bart~\cite{bart} as the backbone network and add a fully connected layer to form the rating classification model. We compare the model's performance given the original text inputs as well as the personalized inputs as follows: 
\begin{equation*}
\begin{aligned}
    &\text{Original Text Input := `<review content>'},\\
    &\text{Personalized Text Input := `user<user ID>|<review content>'}.
\end{aligned}
\end{equation*}
From the evaluation results shown in Figure~\ref{id_figure}, we can observe that compared with the inputs without user ID, the inputs with user ID averagely increase the inference accuracy by 3.94\%. The major reason is that personal information helps to reduce bias in ratings from different users.

\begin{figure}[!t]
\centering
  \includegraphics[width=0.95\columnwidth]{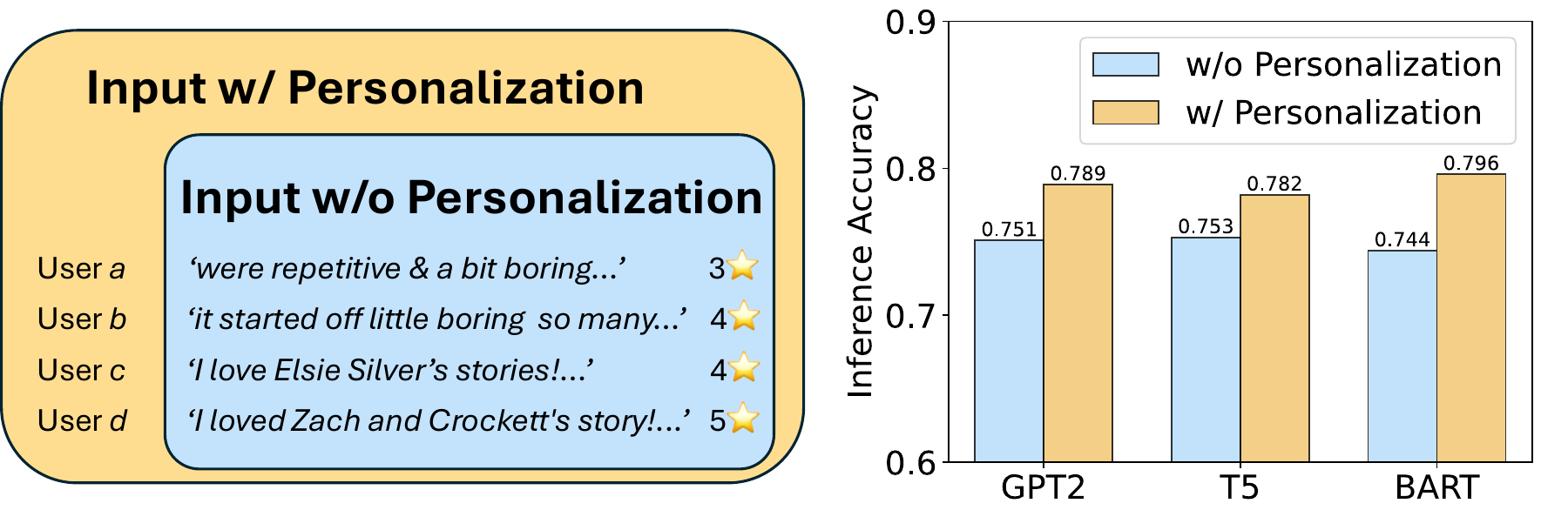}
  \caption{Rating classification accuracy over Amazon-Kindle dataset with and without user ID.}
  \label{id_figure}
\end{figure}



Driven by the application requirement of data anonymization as well as the performance improvement of personal information, we consider how to enable personalized language model learning on the cloud without user ID in text data. 


\subsection{Trivial User Embedding Method}

We first revisit how a language model handles the inputs with user ID, where `user<user ID>' is first processed by a tokenizer, then encoded into token embeddings, and finally fed to upper layers. Considering the practical setting that users do not include user IDs in their uploaded data, a trivial method to obtain personalized user embedding is to offload this task to each mobile device. In particular, a user $n$ learns his/her user embedding $u_n$ on the mobile device with efficient prompt tuning method \cite{prompt_tuning} through optimizing the following objective 
\begin{equation}\label{trivial_local_obj}
    \min_{u_n} \frac{1}{|\mathcal{D}_n|} \sum_{(x,y)\in \mathcal{D}_n} l(h([u_n;x]), y),
\end{equation}
where $\mathcal{D}_n$ denotes user $n$'s local training dataset; $(x,y)$ denotes a sample in the format of (input, target); $[u_n; x]$ denotes the concatenation of the user embedding and the original input; $h$ denotes the language model and can be frozen during the training process; and $l(\cdot,\cdot)$ denotes the loss function. After adding user embedding, the language model's input is
\begin{equation*}
    \text{Model Input := } (\text{user embedding},\ \text{<original input>}).
\end{equation*}

We then consider how to deploy the model with user embedding for real-time inference. (1) For cloud-based model serving, each mobile device needs to upload static user embedding, which functions as a new user ID and does not meet the desired anonymization requirement; and (2) for on-device model serving, the user embedding does not need to be uploaded, maintaining data anonymization. However, the inference latency of language models on resource-constraint mobile devices fails to meet practical real-time requirements. For example, as evaluated in Section~\ref{industrial_result}, the prediction latency of next words or sentences using Qwen1.8B on the Honor V30 Pro testbed is over 83 seconds, far exceeding the latency requirements of mobile keyboard applications.

\subsection{New Design Objectives}

To achieve real-time inference, we still keep language model serving on the cloud. Additionally, to leverage personalized user embedding while keeping data anonymization, we need to break the one-to-one mapping between a user and his/her fixed user embedding. The key new idea is to dynamically sample user embeddings from a user-specific distribution, denoted as $\mathcal{U}_n$ for user $n$. In other words, the anonymous data from a user uploaded to the cloud as the language model's input become
\begin{align*}
    \text{Model Input := } (&\text{user embedding sampled from a local distribution},\\ &\text{<original input>}).
\end{align*}



Under the new framework of generating user embeddings from a local distribution, we formally define design objectives from both personalized model performance and data anonymization, thereby guiding the corresponding requirements on the distribution. 

First, each user's choice of a specific distribution should minimize the loss over local data, formalized as 
\begin{objective}[Personalized Model Performance]\label{obj:person}
The optimization objective for user $n$'s local distribution $\mathcal{U}_n$ for generating personalized user embeddings is 
\begin{equation}\label{local_obj}
    \min_{\mathcal{U}_n} \frac{1}{|\mathcal{D}_n|} \sum_{(x,y)\in \mathcal{D}_n} \mathbb{E}_{u_n\sim\mathcal{U}_n} \left[l(h([u_n;x]), y)\right].
\end{equation}
\end{objective}
\noindent This new objective differs from the conventional objective of on-device training for user embedding (i.e., equation \ref{trivial_local_obj}) in identifying a user-specific distribution rather than a static user embedding. 


Second, the embedding distributions chosen by users should also ensure that it is hard for the cloud to identify any individual user based on the uploaded local embeddings. Since all the users anonymously upload their data with randomly sampled embeddings, from the view of the cloud, it can only observe user embeddings generated by a mixture of distributions. We let $N$ denote the number of users and express the mixture of embedding distributions as $\mathcal{U}=\sum_{n=1}^N w_n \mathcal{U}_n$, where $w_n$ denotes user $n$'s weight, which is proportional to the size of user $n$'s truly uploaded samples and is unknown to the cloud, and $\sum_{n=1}^N w_n=1$. Therefore, if the cloud attempts to identify a certain user $n$ given collected user embeddings, it needs to reconstruct the user $n$'s local distribution $\mathcal{U}_n$ from the global distribution $\mathcal{U}$ and further determine which embeddings are generated by the user-specific distribution $\mathcal{U}_n$. We thus define the anonymization of user embedding from either (1) non-identifiability of user-specific distributions from the mixture of distribution in objective \ref{not_id}, intuitively forming one-way user-embedding mapping; or (2) wrongly attributing collected user embeddings to a specific user's distribution, even when all user-specific distributions are known, in objective \ref{not_att_acc}, intuitively forming many-to-one user-embedding mapping. For non-identifiability, its formal definition has been given in     
in the existing work on mixture of models~\cite{finitemixture_id_1,finitemixture_id_2}. In our context, we formulate a distribution with cumulative distribution function (CDF) and let $d$ denote the dimension of a user embedding. We let $\mathcal{F}=\{F(\mathcal{U}_n; \theta) | \forall n \}$ represent the function space collection of any user-specific distribution $\mathcal{U}_n$'s CDF, which is a family of $d$-dimensional CDFs with parameter $\theta$. We thus define non-identifiability in objective \ref{not_id}. 

\begin{objective}[Non-Identifiability of User-Specific Embedding Distribution]\label{not_id}
  The mixture of finite user-specific embedding distributions $\mathcal{U}$ is not identifiable, if it does not have a unique representation as a combination of distributions from $\mathcal{F}$, namely, there exists another weights $\{c_1,c_2,\cdots,c_M\}$ and distribution parameters $\{\theta'_1,\theta'_2,\cdots,\theta'_M\}$ such that \text{CDF of $\mathcal{U}$} can be expressed as  
\begin{equation}
    \sum_{n=1}^N w_n F(\mathcal{U}_n;\theta_n) = \sum_{m=1}^M c_m F(\mathcal{U}_n;\theta'_m),
\end{equation}
where $\{w_1,w_2,\cdots,w_N\}\not=\{c_1,c_2,\cdots,c_M\}$ or $\{\theta_1,\theta_2,\cdots,\theta_n\}\not=\{\theta'_1,\theta'_2,\cdots,\theta'_M\}$, for any permutation of $m$ on $\{1,2,\cdots,M\}$.
\end{objective}


Even if user-specific distributions are identifiable, we also define objective \ref{not_att_acc} to guarantee the high probability of user embedding misattribution.

\begin{objective}[Misattribution of User Embedding]\label{not_att_acc}
Given known user-specific distributions $\forall n, \mathcal{U}_n$ and a user embedding $u_n$ randomly sampled from $\mathcal{U}_n$, the probability of the cloud not attributing $u_n$ to $\mathcal{U}_n$ is  
\begin{equation}
    \Pr\left(\mathcal{U}_n \neq \arg\max_{\mathcal{U}_k}\Pr(\mathcal{U}_k | u_n)\right) \geq 1 - \epsilon,
\end{equation}
where $\Pr(\mathcal{U}_k|u_n)$ denotes the posterior probability of $u_n$ sampled from $\mathcal{U}_k$ according to Bayes' theorem, and $\epsilon$ denotes a small term. 
\end{objective}

In what follows, we first introduce the design to achieve objectives \ref{obj:person} and \ref{not_id} in Section \ref{sec:design:non:ident} and then present the design to achieve objectives \ref{obj:person} and \ref{not_att_acc} in Section \ref{pl_aue}.

\section{Personalized Learning with Non-Identifiable User Embedding}\label{sec:design:non:ident}

\subsection{Algorithm Design}

We first consider the design of on-device training algorithm to guarantee personalized model performance in objective \ref{obj:person}. The presence of the expectation under a sampled random variable from a parameterized distribution makes it challenging to directly optimize equation~\ref{local_obj}. Inspired by variational autoencoder (VAE)~\cite{vae}, we adopt the reparameterization trick, which introduces an auxiliary random variable $\xi$, independent from sampled user embedding $u_n$, to facilitate the computation of the gradient of an expectation. Formally, we let $\theta$ denote the parameters of user-specific local distribution $\mathcal{U}_n$, such that $u_n = g_\theta(\xi)\sim \mathcal{U}_n$, where $\xi\sim p(\xi)$, and $g_\theta(\cdot)$ denotes the mapping function. The gradient is rewritten as
\begin{equation}\label{re-param}
    \nabla_\theta \mathbb{E}_{u_n\sim\mathcal{U}_n} \left[l(h(u_n;x), y)\right] = \nabla_\theta \mathbb{E}_{\xi\sim p(\xi)}\left[l(h(g_\theta(\xi);x), y)\right],
\end{equation}
where the right-hand side is estimated using Monte Carlo methods. During on-device training, the language model $h$ is frozen, and only the parameters $\theta$ of user-specific distribution need to be trained.  


\begin{algorithm}[!t]
  \caption{Personalized Language Model Learning with Non-Identifiable User Embedding}\label{nonid_alg}
  \begin{algorithmic}[1]
    \REQUIRE The number of users $N$; the latest cloud-based language model $h$; user-specific distributions $\{\mathcal{U}_n|n=1,2,\cdots,N\}$ for generating user embeddings.
    \STATE* {\textbf{Training Phase on Mobile Devices and the Cloud}}
    \STATE* {\tt /* Each Mobile Device's Process */}
    \FOR{$n$ from $1$ to $N$ in parallel}
      \STATE Downloads $h$ from the cloud;
      \STATE Based on equation~\ref{re-param}, freezes $h$, and trains the optimal parameters of the local distribution $\mathcal{U}_n$ w.r.t. equation~\ref{local_obj}, the CDF of which is $F(\cdot;\theta_n)\in \mathcal{F}$;
      \FOR{Each local data sample $(x,y)$}
        \STATE Samples $u_n$ from $\mathcal{U}_n$;
        \STATE Anonymously uploads $([u_n; x],y)$ to the cloud;
      \ENDFOR
    \ENDFOR
    \STATE* {\tt /* Cloud's Process */}
    \STATE Receives samples from all users and constructs the enhanced training set with user embeddings $\mathcal{D}^{+}$;
    \STATE Finetunes $h$ over $\mathcal{D}^{+}$;
    \STATE* {\textbf{Real-Time Inference Phase on the Cloud}}
    \STATE* {\tt /* Each Mobile Device's Process */}
    \STATE For user $n$'s original text input $x$, samples $u_n$ from $\mathcal{U}_n$ and uploads $[u_n;x]$ to the cloud;
    \STATE* {\tt /* Cloud's Process */}
    \STATE Returns the inference result $h([u_n;x])$ to the mobile device.
  \end{algorithmic}
\end{algorithm}

We then consider how to choose user-specific embedding distribution $\mathcal{U}_n$ to satisfy objective \ref{not_id}. From \cite{finitemixture_id_2}, we know that the identifiability of a mixture of distributions is closely related to the linear dependence of its function space.

\begin{lemma}[\cite{finitemixture_id_2}]\label{lemma:id_condition}
    A necessary and sufficient condition for the identifiability of all finite mixtures within the family $\mathcal{F}$ is that $\mathcal{F}$ forms a linearly independent set over the field of real numbers.
\end{lemma}
\begin{proof}
   Please refer to Appendix \ref{app:proof:lemmas}.
\end{proof}

\noindent If the elements in $\mathcal{F}$ of user embedding distributions are linearly independent, there exists a unique solution for an observed mixture of collected user embeddings, implying that the user-specific distribution $\mathcal{U}_n$ involved in the mixture can be identified. Conversely, if the elements in $\mathcal{F}$ are linearly dependent, an observed mixture may correspond to infinite solutions, making the identification of $\mathcal{U}_n$ from the mixture infeasible. Therefore, we let each user choose $\mathcal{U}_n$ with the function space being linearly dependent. Typical examples are Beta distribution and Pearson Type VI distribution. 


Based on the above design rationales, we propose the personalized language model learning algorithm with non-identifiable user embedding in Algorithm~\ref{nonid_alg}. In the training phase, each mobile device first downloads the latest model $h$ from the cloud (line 2) and then trains an optimal user-specific distribution $\mathcal{U}_n$ with linearly dependent function space based on equation~\ref{re-param} (line 3).  To facilitate cloud-based training, for each local sample, the mobile device samples a user embedding from $\mathcal{U}_n$ and anonymously uploads the concatenation of the sampled user embedding and original text data to the cloud (lines 5-6). The cloud finetunes the model $h$ over the collected samples from all users to adapt to the new input space (line 7--8). In the real-time inference phase, the mobile device sends the sample with a randomly sampled user embedding, and the cloud returns the inference result (lines 9-10).


\subsection{Algorithm Analysis}\label{nonid_analysis}

We prove the non-identifiability of Algorithm~\ref{nonid_alg}. We instantiate user-specific distribution $\mathcal{U}_n$ with Beta distribution. 

\begin{theorem}\label{nonid_beta}
If each dimension of each user's embedding follows a Beta distribution, the mixture of distributions is non-identifiable.
\begin{proof}
 Please refer to Appendix \ref{app:theorem1:proof}.
\end{proof}
\end{theorem}



We then analyze the efficiency of Algorithm~\ref{nonid_alg}. For on-device offline training, only the parameters of user-specific distribution need to be updated, while the language model is frozen. Therefore, it is resource efficient for mobile devices. For cloud-based online inference, compared with conventional cloud-based model serving without user embedding, Algorithm~\ref{nonid_alg} additionally requires to sample a user embedding, upload the embedding to the cloud, and increase the input length of the cloud-based model. First, each mobile device can offline sample user embeddings, store them locally, and directly fetch them for online inference. Second, uploading latency depends on the dimension of the user embedding. For example, uploading a 768-dimensional user embedding takes only 0.6 millisecond with a network bandwidth of 5 MB/s. Third, adding a user embedding is equivalent to adding just one token to the model input, resulting in little increase in forwarding latency. Overall, Algorithm~\ref{nonid_alg} strictly satisfies the real-time requirement of cloud-based model inference. 

\section{Personalized Learning with Misattributed User Embedding}\label{pl_aue}

\subsection{Algorithm Design}




To avoid accurate user embedding attribution, the key intuition is that if the distances between different users' distributions (i.e., $\forall n, \mathcal{U}_n$) are small, it is hard to determine the source distribution of a sampled embedding according to its posterior probability. The key difference from the non-identifiable design in Algorithm~\ref{nonid_alg} is that we do not need to limit the distribution type and just require appropriately setting of on-device training hyper-parameters. Specifically, during on-device training process (line 3), we propose to let each mobile device train $\mathcal{U}_n$ from the same initialization and bound local updates by using a small learning rate and applying gradient clipping techniques. In detail, user $n$ first initializes the parameter $\theta_n$ over the local distribution $\mathcal{U}_n$ with $\theta$ shared by all users. Then, user $n$ iteratively updates the parameter $\theta_n$ of local distribution $\mathcal{U}_n$ using gradient descent as
\begin{equation}
    \theta_n = \theta_n - \eta G({\hat{\nabla}_{\theta_n}}),
\end{equation}
where $\eta$ denotes the learning rate, $\hat{\nabla}_{\theta_n}$ denotes the observed gradient, and $G(\cdot)$ denotes the gradient clipping function. 


\subsection{Algorithm Analysis} 

We analyze the misattribution probability of user embedding. We instantiate user-specific distribution  $\mathcal{U}_n$ with commonly used multi-dimensional Gaussian, namely, $\mathcal{U}_n=\mathcal{N}(\mathbf{\mu}_n, \Sigma_n)$, where the distribution parameters $\mathbf{\mu}_n\in \mathbb{R}^d$ and $\Sigma_n\in \mathbb{R}^{d\times d}$ denote mean and covariance, respectively. The covariance $\Sigma_n$ is also a diagonal matrix, because different dimensions of a user embedding should be independent from each other. During on-device training, $\mu_n$ is trainable, and $\Sigma_n$ is fixed at $\sigma^2\mathbf{I}$. To simplify analysis, we make assumptions on event independence and on-device training as follows.

\begin{assumption}\label{assumption:id}
    For $u_n \sim \mathcal{U}_n$, and for any $i,j\not=n$, the event of $\Pr(u_n | \mathcal{U}_i) \le \Pr(u_n|\mathcal{U}_n)$ and the event of $\Pr(u_n|\mathcal{U}_j) \le \Pr(u_n|\mathcal{U}_n)$ are independent.
\end{assumption}

\begin{assumption}\label{assumption:iteration}
    Each user performs at most $T$ local iterations.
\end{assumption}

\begin{assumption}\label{assumption:g_norm}
    During the local training of any user $n$, the $l_2$-norm of the clipped gradient $G(\hat{\nabla}_{\mu_n})$ is always bounded by $G^2$.
\end{assumption}

We then have the following theoretical result.

\begin{theorem}\label{gaussian_id}
    Under Assumptions~\ref{assumption:id}, \ref{assumption:iteration}, and~\ref{assumption:g_norm}, for a user embedding $u_n$ sampled from $\mathcal{U}_n$, when the prior $\Pr(\mathcal{U}_k)$ are the same for any $k$,
    \begin{equation}
    \Pr\left(\mathcal{U}_n \neq \arg\max_{\mathcal{U}_k}\Pr(\mathcal{U}_k | u_n)\right) \geq 1 - \left(\Phi(\frac{\parallel \eta TG\parallel}{\sigma})\right)^{N-1},
   \end{equation}
    where $\Phi(\cdot)$ denotes the CDF of the standard Gaussian distribution.   
\end{theorem}


\begin{proof}
    Please refer to Appendix \ref{app:proof:theorem2}.
\end{proof}


Theorem~\ref{gaussian_id} indicates that to increase the misattribution probability of user embedding, thereby enhancing data anonymization against the cloud, each mobile device can reduce the learning rate $\eta$, decrease the number of local iterations $T$, lower the $l_2$-norm of the clipped gradient $G^2$, or increase the variance $\sigma$.

\begin{table*}[!t]
\caption{IDfree-PL vs. Baselines from inference accuracy with different language models over public datasets.}
\centering
\resizebox{\linewidth}{!}{
\begin{tabular}{cccccccccc}
\toprule
\multirow{2}{*}{Dataset} & \multirow{2}{*}{Model} & \multirow{2}{*}{Cloud w/o ID} & \multirow{2}{*}{On-Device} & \multicolumn{2}{c}{IDfree-PL} & \multicolumn{2}{c}{vs. Cloud w/o ID} & \multicolumn{2}{c}{vs. On-Device} \\
\cmidrule(r){5-6}\cmidrule(r){7-8}\cmidrule(r){9-10}
&&&& Beta & Gaussian & Beta & Gaussian & Beta & Gaussian \\
\midrule
\multirow{3}{*}{Sentiment140} & GPT2 & 80.00\% & 81.20\% & 83.36\%  & 83.48\% & +3.36\% & +3.48\% & +2.16\% & +2.28\%\\
 & T5 & 79.87\% & 81.75\% & 80.90\% & 80.92\% & +1.03\% & +1.05\% & -0.85\% & -0.83\%\\
 & Bart & 81.64\% & 82.00\% & 84.18\% & 84.30\% & +2.54\% & +2.66\% & +2.18\% & +2.30\%\\
 \cmidrule{2-10}
\multirow{3}{*}{Amazon} & GPT2 & 75.11\% & 78.28\% & 79.16\% & 78.86\% & +4.05\% & +3.75\% & +0.88\% & +0.58\%\\
 & T5 & 75.31\% & 80.02\% & 79.21\% & 78.87\% & +3.90\% & +3.67\% & -0.81\% & -1.15\%\\
 & Bart & 74.42\% & 78.18\% & 80.11\% & 79.96\% & +5.69\% & +5.54\% & +1.93\% & +1.78\%\\
 \cmidrule{2-10}
Reddit & GPT2 & 22.41\% & 19.51\% & 22.87\% & 23.05\% & +0.46\% & +0.64\% & +3.36\% & +3.54\% \\
\bottomrule
\end{tabular}
}
\label{tab:performance_comparison}
\end{table*}

\section{Evaluation on Public Datasets}
\subsection{Evaluation Setups}\label{eval_setup}
\begin{figure*}[!t]
\centering

\subfigure{\includegraphics[width=0.6\linewidth]{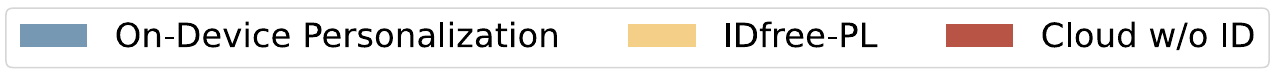}}\\
\vspace{-1.39em}
\begin{minipage}{0.99\linewidth} 
\centering
\subfigure[Sentiment140]{
\includegraphics[height=0.19\textwidth]{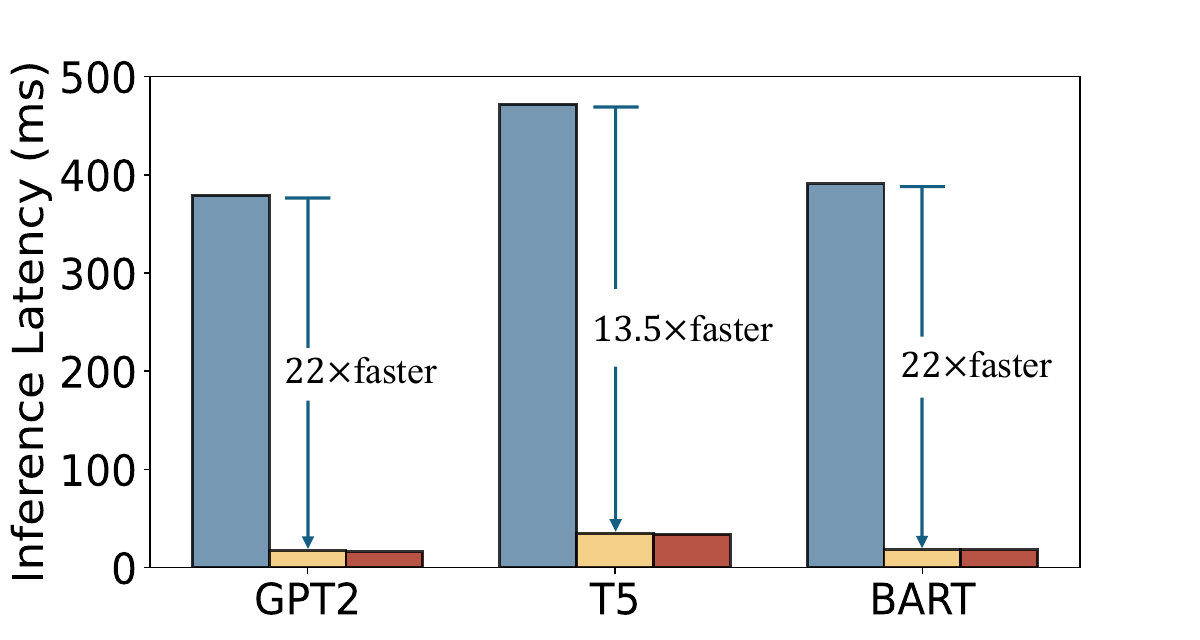}
}
\subfigure[Amazon]{
\includegraphics[height=0.19\textwidth]{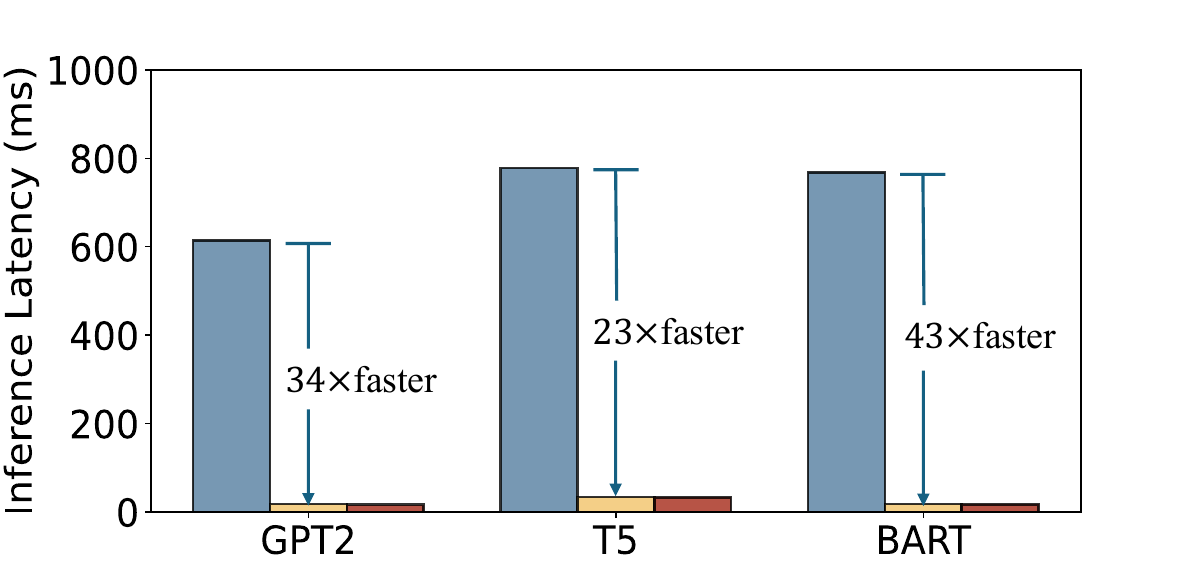}
}
\subfigure[Reddit]{
\includegraphics[height=0.19\textwidth]{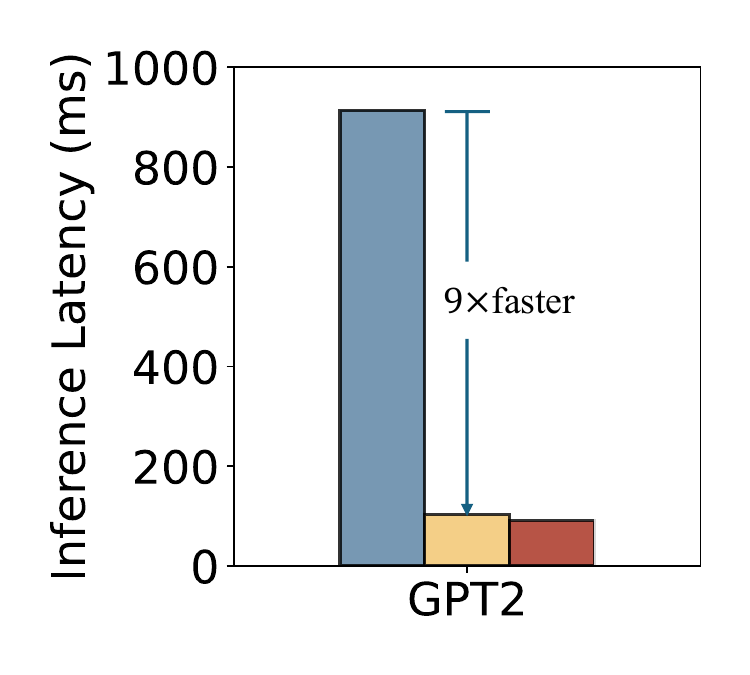}
}
\vspace{-1em}
\caption{IDfree-PL vs. Baselines from inference latency with different language models over different datasets.}
\label{fig:latency}
\end{minipage}
\end{figure*}

\textbf{Datasets and Language Tasks.} We first introduce three public datasets for different language tasks and describe data preprocessing details. (1) \textbf{Sentiment140 dataset for sentiment classification task}~\cite{sentiment140}: It comprises 1,600,000 tweets and corresponding sentiments (negative or positive) from 659,775 Twitter users. We naturally partition this dataset by letting each Twitter account correspond to a user. We keep only the users who hold more than 100 samples and get 163 users in total. For the split of each user’s training and test sets, about 80\% of the samples, which are with the timestamps no more than ``2009-06-06 23:53:52'' into the training set and take the remaining samples into the test set. 
(2) \textbf{Amazon-Kindle dataset for sentiment classification task}~\cite{amazon}: It comprises 25,600,000 reviews and corresponding ratings from 5,600,000 users. The ratings range from 0 to 5. We label the ratings no more than 3 as negative, the ratings equal to 4 as neutral, and the ratings equal to 5  as positive. We naturally partition this dataset by letting each account correspond to a user. To ensure that each user has sufficient data, we select 2,000 users with the most samples. By further filtering users who scored in fewer than three categories to eliminate data with low quality from ``paid reviewers'', we obtain 1,435 users. For the split of training and test sets, about 80\% of the samples, which are with the timestamps no more than 1,632,712,835,970 fall into the training set, while the rest falls into the test set. (3) \textbf{Reddit dataset for next word generation task}~\cite{leaf}: It comprises 56,587,343 comments from 1,660,820 reddit users. We naturally partition this dataset by letting each Reddit account correspond to a user. We keep only the users who hold 300 -- 320 comments and get 1,407 users in total. About 75\% of the samples fall into the training set, and the other samples fall into the test set.

\begin{figure*}[!t]
\centering

\begin{minipage}{0.99\linewidth} 
\centering
\subfigure{
\includegraphics[width=0.085\textwidth]{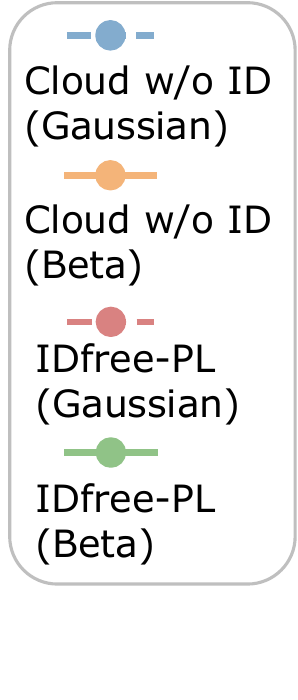}
}
\renewcommand{\thesubfigure}{(a)}
\subfigure[GPT2]{
\includegraphics[width=0.28\textwidth]{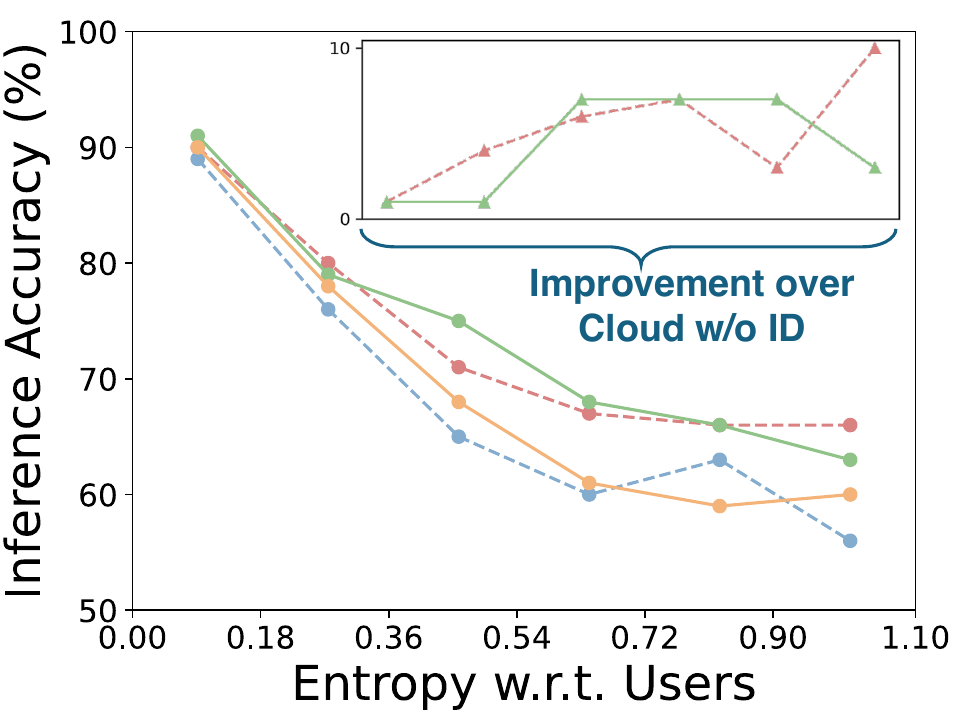}
}
\renewcommand{\thesubfigure}{(b)}
\subfigure[T5]{
\includegraphics[width=0.28\textwidth]{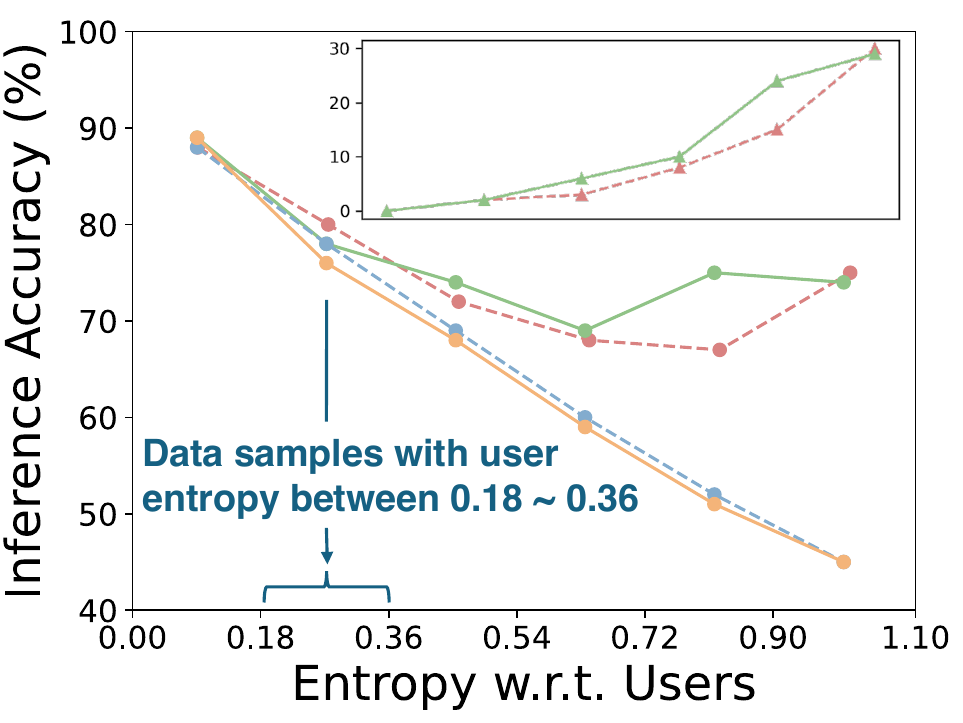}
}
\renewcommand{\thesubfigure}{(c)}
\subfigure[Bart]{
\includegraphics[width=0.28\textwidth]{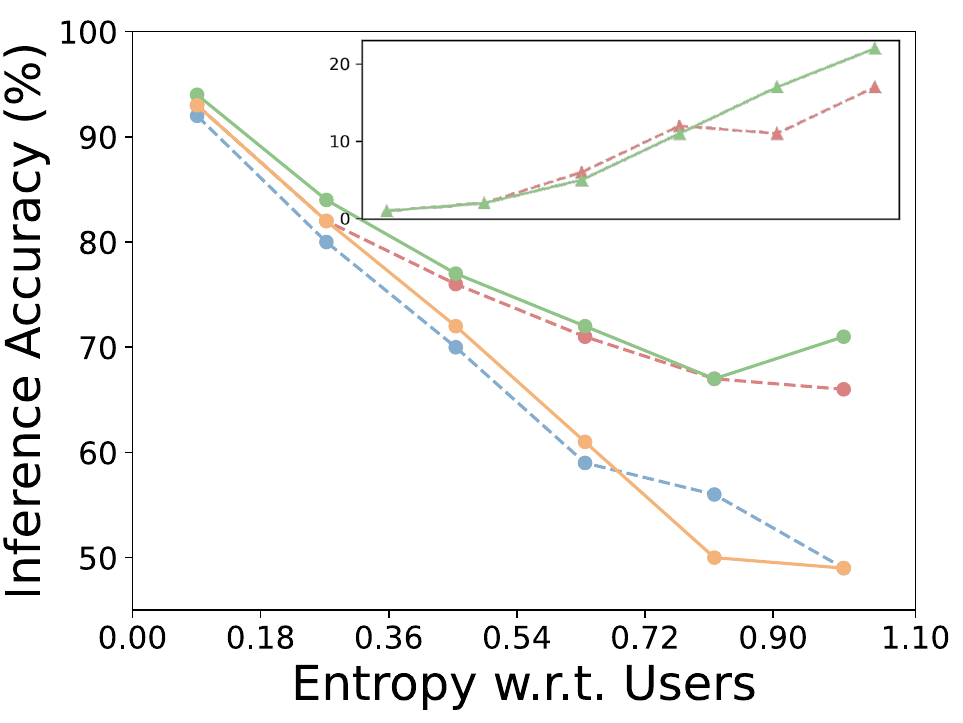}
}
\vspace{-1em}
\caption{Inference accuracy of test samples with varying sensitivity to user embeddings over Amazon-Kindle dataset.}\label{person_data}
\end{minipage}
\end{figure*}

\textbf{Models.} For the sentiment classification task over Sentiment140 and Amazon-Kindle datasets, we take three representative language models for evaluation, including the decoder-only GPT2-base from OpenAI~\cite{gpt2} with approximately 124 million parameters, encoder-decoder T5-base from Google~\cite{t5} with around 220 million parameters, and encoder-decoder Bart-base from Meta~\cite{bart} with about 130 million parameters. For the next word generation task over Reddit dataset, we take the decoder-only GPT2 for evaluation.

\textbf{Baselines.}  (1) {\bf Cloud-Based Learning Without User ID} ({\bf Cloud w/o ID}) \cite{ChineseGPT}, which trains the language model over the training set without personal information. This baseline is currently deployed in mobile WeChat IME app, guarantees data anonymization, and is introduced to validate the necessity of personalization. (2) {\bf On-Device Learning for Model Personalization} ({\bf On-Device}) \cite{deeptype, MPDA}, which lets each mobile device download the cloud-based model and finetune over the local training data. Each user's personalized model needs to be deployed on the mobile device for inference, because it is unaffordable for the cloud to maintain a large number of user-specific models. This baseline does not upload the user data. 



\textbf{Implementation.} We implement the personalized learning design with non-identifiable user embedding using Beta distribution, called {\bf IDfree-PL (Beta)}. We also implement the design with misattributed user embedding using multidimensional Gaussian, called {\bf IDfree-PL (Gaussian)}. To ensure high misattribution probability, we set the on-device learning rate to $1\times10^{-3}$, set the max norm of gradient clipping to 5, and set the variance to 0.2 as default. 
 
For cloud-based training, we adopt Adam with weight decay (AdamW) as the optimizer. We set the learning rate to $5\times 10^{-5}$ and employ a linear scheduler to maintain training stability. We train for 10 epochs with a batch size of 512. For on-device training, we set the user embedding dimension to 768, matching the hidden size of GPT-2, T5, and Bart. We use a batch size of 32 and train for 15 epochs as default. We use one NVIDIA V100 GPU to evaluate inference latency on cloud server, and test the inference latency of the on-device personalized model baseline using the Honor V30 Pro (smartphone) which contains an 8-core Kirin 990 CPU. 

For the implementation of the personalized inference in IDfree-PL, as the user embeddings are random vectors and are not easily transferred into input tokens, which may not be compatible with the sequence encoding in the transformers, we rewrite the forward pass of the embedding layer. The modified embedding layer outputs the concatenation of the user embedding and raw text embedding. Specifically, for the sentiment classification task, to increase the influence of the user embedding on the classification result, we place the user embedding after the raw text embedding and use the transformer's last hidden state of the top layer as the input to the classification head. For the next word generation task, following the idea of prompt tuning~\cite{prompt_tuning}, we place the user embedding before the raw embedding result to achieve personalized generation results.

\subsection{Evaluation Results}
\label{sec:evaluate:results}


\textbf{Inference Accuracy and Latency.} We show in Table~\ref{tab:performance_comparison} and Figure~\ref{fig:latency} inference accuracy and latency. First, compared with cloud-based learning without user ID, we can observe that IDfree-PL (Beta) and IDfree-PL (Gaussian) averagely increase accuracy by 3.00\% and 2.95\%, respectively, while adding only 0.01s of inference latency. We can draw that personalization with anonymous user embeddings in our design can indeed improve inference accuracy while maintaining efficiency. Second, compared with on-device model personalization, IDfree (Beta) and IDfree (Gaussian) averagely increase accuracy by 1.26\% and 1.21\%, respectively. The improvement of IDfree-PL is mainly due to mitigating overfitting with large-scale samples on the cloud uploaded from many users. In contrast, on-device training suffers from the scarcity of local data samples. Specifically, for the complex next word generation task over Reddit dataset, the input-to-output mapping space is larger and each Reddit user's data are sparser compared to the classification tasks, leading to a higher generalization error in on-device training. Consequently, the accuracy of on-device personalized model is even lower than that of cloud-based model without personalization due to severe over-fitting. Regarding inference latency, IDfree-PL with cloud-based model serving is up to $43\times$ faster than personalized model serving on the resource-constrained mobile device.




\textbf{Impact of Personalization.} We experimentally investigate what types of data samples benefit from personalized learning. Intuitively, some general samples conform to universal behavior patterns across users, thus concatenating embeddings from different users may not affect the model outputs. In contrast, user-specific samples are sensitive to user embeddings, indicating that concatenating embeddings from different users significantly changes the model output. Based on this intuition, we study the relationship between the accuracy of samples and their sensitivity to user embeddings, and plot the statistical results over Amazon-Kindle dataset in Figure~\ref{person_data}. For each test sample, we randomly select 20 users and concatenate the raw input with the user embeddings from these users. We then compute the entropy of the 20 predictions as a measurement of the sample's sensitivity to user embeddings, called ``user entropy''. We divide the entropy into six intervals and compute the inference accuracy for samples within each interval. From Figure~\ref{person_data}, we can see that the improvement of IDfree-PL over the cloud-based learning without user ID is generally positively correlated with the sensitivity of the samples. For samples with the lowest sensitivity, IDfree-PL improves accuracy by an average of 0.56\%. In contrast, for samples with the highest sensitivity, IDfree-PL boosts accuracy by an average of 19.22\%. These results indicate that IDfree-PL effectively identifies samples that benefit from personalization, namely, those sensitive to user embeddings and with high user entropy, thereby increasing overall inference accuracy by enhancing performance on these sensitive samples.



\textbf{Trade-off between Privacy and Personalization.} We evaluate IDfree-PL (Gaussian) of setting different variances to illustrate the trade-off between privacy and personalization. We show the inference accuracy and misattribution probability of IDfree-PL (Gaussian) with different models over the Amazon dataset in Table~\ref{tradeoffacc}. In particular, we randomly sample 1,000 user embeddings and compare the distance between the sampled embeddings and the means of the users' Gaussian distribution for the estimation of the misattribution probability. From Table~\ref{tradeoffacc}, we can see that a larger variance enhances privacy but reduces personalization, which is consistent with Theorem~\ref{gaussian_id}. More specifically, when the variance is set to $0.0^2$, which indicates that the user embeddings are static, IDfree-PL achieves the averaged inference accuracy of 80.0\% but the misattribution probability is 0.0\%. When the variance is set to $0.3^2$, IDfree-PL achieves the averaged inference accuracy of 78.77\%, while the misattribution probability is 70.93\%. These evaluation results demonstrate that a large variance increases the overlap of different user distributions, thereby enhancing privacy protection, at the cost of reducing the inference accuracy by 1.23\%.

\begin{table}[!t]
\caption{Accuracy (ACC) and misattribution probability (MIS) of IDfree-PL (Gaussian) with different variance (Var).}\label{tradeoffacc}
\centering
\resizebox{0.99\linewidth}{!}{
\begin{tabular}{ccccccccc}
\toprule
\multirow{2}{*}{Model} & \multicolumn{2}{c}{Var=$0.0^2$} & \multicolumn{2}{c}{Var=$0.1^2$} & \multicolumn{2}{c}{Var=$0.2^2$} & \multicolumn{2}{c}{Var=$0.3^2$} \\
\cmidrule(r){2-3}\cmidrule(r){4-5}\cmidrule(r){6-7}\cmidrule(r){8-9}
& ACC & MIS & ACC & MIS & ACC & MIS & ACC & MIS \\
\midrule
GPT2 & 79.9\% & 0.0\% & 79.3\% & 11.4\% & 78.9\% & 61.9\% & 78.3\% & 83.3\% \\
T5 & 79.5\% & 0.0\% & 79.2\% & 2.1\% & 78.9\% & 41.5\% & 78.1\% & 72.4\% \\
Bart & 80.6\% & 0.0\% & 79.7\% & 1.9\% & 80.0\% & 25.3\% & 79.9\% & 57.1\% \\
\bottomrule
\end{tabular}
}
\end{table}

\begin{figure}[!t]
\centering
\centering
\includegraphics[width=0.95\columnwidth]{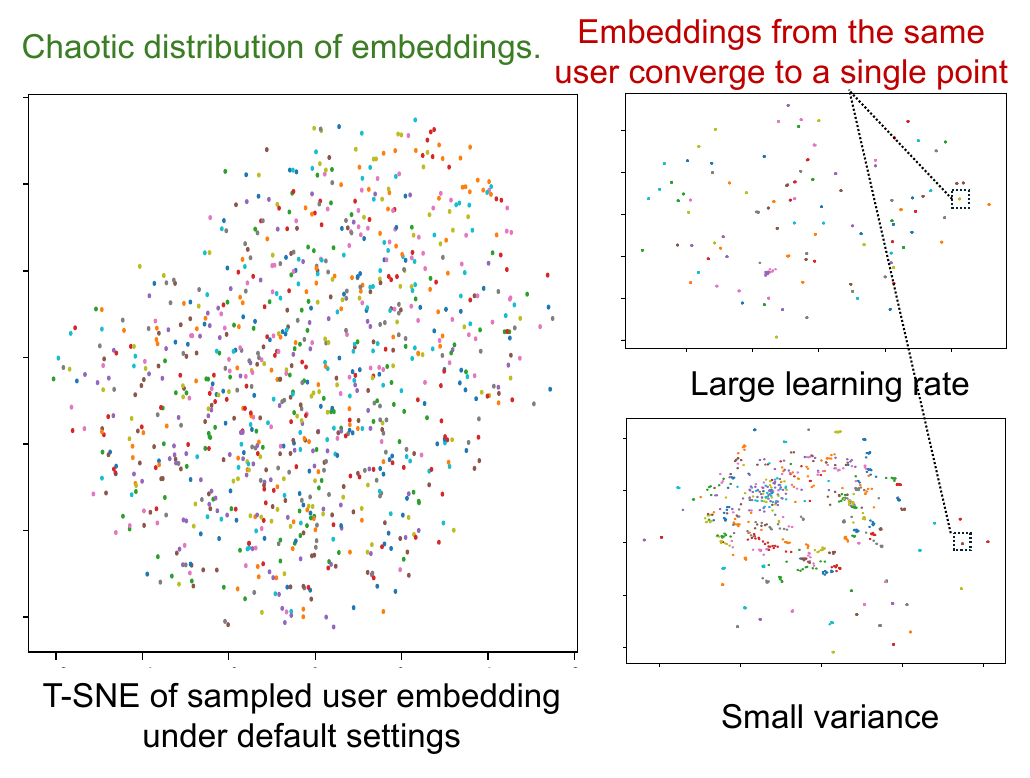}
\vspace{-1em}
\caption{Visualization of user embeddings.}\label{tsne_result}
\end{figure}

\textbf{Visualization of Anonymous User Embeddings.} We visualize user embeddings to intuitively demonstrate data anonymization guarantee. We take the user embedding distributions from IDfree-PL (Gaussian) trained with GPT2 on Amazon-Kindle dataset for illustration. More visualization results are put in Appendix~\ref{app:tsne}. We sample 10 data points from each distribution. We then perform t-distributed stochastic neighbor embedding (T-SNE) on all the sampled points and plot the results after dimension reduction in Figure~\ref{tsne_result}. We represent embeddings sampled for the same user using the same color and show results from 100 users to reduce clutter. For comparison, we also plot the results in inappropriate settings, such as an excessively large on-device learning rate (i.e., 0.01) and a very small variance (i.e., 0.05). From Figure~\ref{tsne_result}, we can observe that under the default appropriate settings, the embeddings from different users are chaotically mixed together, well preserving data anonymization, while the inappropriate settings cause the embeddings from the same user being too close and almost converges to a single point, suffering from accurate user embedding attribution. 


\begin{table}[!t]
    \caption{Typical examples to illustrate how personalization works in the review rating classification task. \textcolor{blue}{Negative} descriptions are colored in blue, and \textcolor{red}{positive} descriptions are colored in red.}\label{tab:case_study}
    \centering
    
    \begin{tabular}{|ll|}
    \hline
    \textbf{Review 1.} Prediction w/o ID: $\leq 3$ & Personalized prediction: 4\\
    \multicolumn{2}{|l|}{\textit{Yes, I know lord Peter is smart, but does he \textcolor{blue}{have to recite limitless}}} \\
    \multicolumn{2}{|l|}{\textit{poetry and talk \textcolor{blue}{in ways few understand?} ...}}  \\
    \textbf{Review 2.} Prediction w/o ID:   $4$ & Personalized prediction: $\leq 3$\\
    \multicolumn{2}{|l|}{\textit{I \textcolor{red}{definitely liked the story} but have \textcolor{blue}{never liked stories} even in}}\\
    \multicolumn{2}{|l|}{\textit{series that end on such a totally sad note.}}\\
    \textbf{Review 3.} Prediction w/o ID:   $5$ & Personalized prediction: 4\\
    \multicolumn{2}{|l|}{\textit{... I \textcolor{red}{love this series so much} ... I \textcolor{red}{loved the romance} between Axel}}\\
    \multicolumn{2}{|l|}{\textit{and Nora...I \textcolor{red}{cannot wait for} Dirty Groom, the next book...}}  \\
    \hline
    \end{tabular}
\end{table}

{\bf Case Study on Personalized Rating.} We present three representative reviews in Table \ref{tab:case_study} to illustrate how personalization influences inference. Review 1 expresses a negative sentiment about the book, yet the user rates it 4 points. This is consistent with the user's tendency to rate less severe negative reviews as 4 points, and preference for Sayers' stories featuring Lord Peter. Review 2 is ambiguous, containing both positive and negative elements, which makes it challenging to rate. However, considering the user's past reviews, it is likely the user will assign a lower rating. Review 3 is highly positive, but the user rates it only 4 points. This behavior aligns with the user's pattern of giving positive ratings to most books and reserving 5 points for only a few exceptional ones.

\section{Evaluation on Industrial Dataset}


\subsection{Mobile Chinese Keyboard Application}

Chinese virtual keyboard (e.g., iFLYTEK IME\footnote{https://srf.xunfei.cn/} and Sogou IME\footnote{https://shurufa.sogou.com/}), involves two important tasks that have already been widely deployed in practical mobile applications~\cite{Gboard, Baidu_IME_study, ChineseGPT, intelassoc}. One task is Pinyin Input Method Editor (PinyinIME), which transforms Chinese pinyin sequences to Chinese characters. The other task is Intelligent Association (IntelAssoc), which predicts possible next words or sentences based on the context already entered by the user to improve input efficiency. 


\subsection{Evaluation Setups}

We build the dataset based on the real user posts on an online social network, \textit{Sina Weibo}\footnote{https://m.weibo.cn/}. We randomly select 400 active users in past 3 years and collect all their posts by an open-source Weibo dataset crawler. We then conduct Chinese segmentation into several pieces to simulate the user's input units, and further convert Chinese characters into perfect pinyin syllables. To reflect the data distribution of personalized input styles, we generate the pinyin input and their corresponding Chinese characters through: (1)~text segmentation at multiple granularities to simulate different input lengths preferred by users; (2)~pinyin input style selection~\cite{intelassoc} among ``perfect Pinyin'', ``abbreviated Pinyin'', and ``typo Pinyin'' with user-specific probabilities. We select about 80\% of samples based on timestamps as the training set and reserve the remaining for testing. 



For the PinyinIME task, context and pinyin input of raw sample are taken as the model input to predict the target word or sentence. For the IntelAssoc task, only context is taken as the model input to predict the target word or sentence. We take GPT2 and Qwen1.8B for PinyinIME and IntelAssoc, respectively. We adopt the similar evaluation settings as those on the public datasets. The major difference is that we use Bfloat16 to reduce the memory consumption for all the experiments with Qwen1.8B.


\begin{table}[!t]
\caption{IDfree-PL vs. Baselines from inference latency on industrial datasets.}\label{tab:latency_industrial}
\vspace{-1em}
\centering
\resizebox{0.95\linewidth}{!}{
\begin{tabular}{ccccc}
\toprule
{Task} & {Model} & Cloud w/o ID  & On-Device & IDfree-PL \\
\midrule
PinyinIME & GPT2 & 40ms  & 882ms & 43ms\\
IntelAssoc & Qwen1.8B & 86ms  & 83960ms & 87ms \\
\bottomrule
\end{tabular}
}
\end{table}

\begin{table}[!t]
\caption{IDfree-PL vs. Real-Time Serving Baseline from top-1 inference accuracy on industrial datasets.}
\vspace{-1em}
\centering
\resizebox{0.99\linewidth}{!}{
\begin{tabular}{ccccc}
\toprule
\multirow{2}{*}{Task} & \multirow{2}{*}{Model} & \multirow{2}{*}{Cloud w/o ID} & \multicolumn{2}{c}{IDfree-PL} \\
\cmidrule(r){4-5}
 & & & Beta & Gaussian\\
\midrule
PinyinIME & GPT2 & 79.11\% & 79.90\% (+0.79\%) & 79.90\% (+0.79\%)\\
IntelAssoc & Qwen1.8B & 1.38\% & 2.95\% (+1.57\%)& 3.05\% (+1.67\%)\\
\bottomrule
\end{tabular}
}
\label{tab:performance_industrial}
\end{table}

\subsection{Evaluation Results}~\label{industrial_result}

We first present the inference latency of IDfree-PL and the baselines in Table~\ref{tab:latency_industrial}. We observe that compared to the cloud-based model without user ID, IDfree-PL increases inference latency by only up to 3 milliseconds. This ensures users receive responses within 50 milliseconds for PinyinIME task and 100 milliseconds for IntelAssoc task. In contrast, on-device personalized model inference takes over 800 milliseconds with GPT2 and 80 seconds with Qwen1.8B, which are impractical for mobile keyboard applications.


We then report the top-1 inference accuracy in Table~\ref{tab:performance_industrial}. Compared to the cloud-based model without user ID, IDfree-PL (Beta) and IDfree-PL (Gaussian) increase the accuracy by 1.18\% and 1.23\%, respectively. We note that the accuracy of IntelAssoc task is typically low due to the vast linguistic space of Chinese~\cite{intelassoc} and the strict requirement for an exact match between predicted results and true labels when calculating accuracy.

\section{Related Work}


We briefly review the related work on personalized learning from the perspective of different data settings in practical applications. 

The first line of work allows the cloud to collect users' data as well as their user ID or personal information. One typical application scenario is recommender system, which collects user profiles (e.g., user ID, gender, and age) and user behaviors (e.g., viewed goods, categories, and shops) ~\cite{DeepFM, EDCN, PNN, DIN, PLE, Explainable_Recommendation}. In general, the user profile is transformed into a lower-dimensional vector through a sparse user embedding layer~\cite{DIN}, which is then fed to the upper dense layers along with the embedded user behaviors. In the training phase, all model parameters including the user embedding layer, are trained end-to-end over the collected dataset. Another application scenario is personalized dialogue system \cite{DBLP:conf/acl/LiGBSGD16, DBLP:conf/emnlp/ChanLYCHZY19, DBLP:conf/naacl/WuMY21, RECAP}, where users' historical conversations are collected to enhance future model responses. Previous work ~\cite{DBLP:conf/acl/LiGBSGD16, DBLP:conf/emnlp/ChanLYCHZY19} used historical conversations to train continuous user embeddings. More recent work employed retrieval augmented generation (RAG) to enhance model response~\cite{RECAP}. These work all required explicit user ID for the cloud to record the user's historical data, and the embedding of a single user or item is static.

Another line of work operates under strict privacy regulations, assuming that all user data fields must remain unuploaded and unexposed to the cloud. Consequently, this line of work did not need to consider data anonymization. The most celebrated framework is federated learning (FL)~\cite{Fedavg, scaffold}. The training process of FL is under the coordination of a parameter server, which maintains a global model and selects mobile devices to perform local training and upload model updates. Google has deployed FL in their mobile keyboard application, called Gboard~\cite{Gboard, Gboard2, Gboard3}. To further improve model personalization effect, personalized federated learning (PFL)~\cite{perFLvbi, perFLshare, perFLdisentangle} was proposed, maintaining user-specific models on mobile devices. PFL focused on effectively and efficiently adapting the global model to the local data using techniques like parameter decoupling~\cite{perFLloc}, meta learning~\cite{perflmeta}, and model interpolation~\cite{perFLadapt}. However, both training and inference phases need to be executed on each resource-constrained mobile device. Therefore, the deployed language models on Gboard were quite light-weight, such as a 1.4MB Recurrent Neural Network (RNN) for next word prediction \cite{Gboard}. For Transformer-based complex language models, as shown in Section~\ref{sec:evaluate:results} and Section~\ref{industrial_result}, on-device inference inevitable breaks the real-time serving requirement. 

Parallel to existing work, this work focuses the new data setting that the cloud collects user data without any identifier. We further propose IDfree-PL to train a personalized language model on the cloud. In addition, the model inference phase of IDfree-PL is still executed on the cloud, thereby satisfying real-time serving requirement in practice even for complex language models. 

\section{Conclusion}

In this work, we have identified a new application requirement of learning a personalized language model on the cloud over the text data anonymously uploaded from mobile devices. We have proposed IDfree-PL, which dynamically samples embeddings from user-specific distributions trained on local user data to achieve personalization, while breaking the one-to-one mapping between users and embeddings to guarantee data anonymization. Evaluation results have demonstrated the effectiveness and the efficiency of IDfree-PL over cloud-based learning without personal information and on-device personalized model learning. 


\section*{Acknowledgements}
This work was supported in part by National Key R\&D Program of China (No. 2022ZD0119100), China NSF grant No. 62025204, No. 62202296, No. 62272293, and No. 62441236, Tencent WeChat Research Program, and SJTU-Huawei Explore X Gift Fund. The opinions, findings, conclusions, and recommendations expressed in this paper are those of the authors and do not necessarily reflect the views of the funding agencies or the government.

\clearpage
\bibliographystyle{ACM-Reference-Format}
\bibliography{main}

\clearpage
\appendix
\newpage
\section{Analysis and Proofs}

\subsection{Proof of Theorem~\ref{nonid_beta}}\label{app:theorem1:proof}

\begin{proof}[Proof of Theorem~\ref{nonid_beta}]
Let $f=\sum_{n=1}^N w_n f(u_n;\alpha_n, \beta_n)$ denote the probability density function (PDF) of the mixture of the embeddings from users.

We first study the dimension $i$ of user $n$'s embedding, the PDF of which is $f(u_n^i;\alpha_n^i, \beta_n^i)$.
As
\begin{equation}
\begin{aligned}
&f(u_n^i;\alpha_n^i+1, \beta_n^i) = \frac{{(u_n^i)}^{\alpha_n^i}(1-{u_n^i})^{\beta_n^i-1}}{B(\alpha_n^i+1,\beta_n^i)},\\
and\ &f(u_n^i;\alpha_n^i, \beta_n^i+1) = \frac{{(u_n^i)}^{\alpha_n^i-1}(1-{u_n^i})^{\beta_n^i}}{B(\alpha_n^i,\beta_n^i+1)},
\end{aligned}
\end{equation}
where $B(\alpha,\beta)=\frac{\Gamma(\alpha)\Gamma(\beta)}{\Gamma(\alpha+\beta)}$ and $\Gamma(\cdot)$ is the gamma function. 

We then have
\begin{equation}
f(u_n^i;\alpha_n^i, \beta_n^i) = c_n^i(1) f(u_n^i;\alpha_n^i+1, \beta_n^i) + c_n^i(2) f(u_n^i;\alpha_n^i, \beta_n^i+1),
\end{equation}
where $c_n^i(1)=\frac{B(\alpha_n^i+1,\beta_n^i)}{B(\alpha_n^i,\beta_n^i)}$ and $c_n^i(2) = \frac{B(\alpha_n^i,\beta_n^i+1)}{B(\alpha_n^i,\beta_n^i)}$.

We further simplify $c_n^i(1), c_n^i(2)$.
\begin{equation}\label{decomp}
c_n^i(1)=\frac{B(\alpha_n^i+1,\beta_n^i)}{B(\alpha_n^i,\beta_n^i)}=\frac{\Gamma(\alpha_n^i+1)\Gamma(\alpha_n^i+\beta_n^i)}{\Gamma(\alpha_n^i)\Gamma(\alpha_n^i+\beta_n^i+1)}\overset{(a)}{=}\frac{\alpha_n^i}{\alpha_n^i+\beta_n^i},
\end{equation}
where (a) follows that $\Gamma(z+1)=z\Gamma(z)$ for any $z$. Similarly, we have $c_n^i(2)=\frac{\beta_n^i}{\alpha_n^i+\beta_n^i}$ with $c_n^i(1)+c_n^i(2)=1$, indicating that the PDF of user $n$'s embedding's $i$-th dimension can be written as the summation of two other PDFs.

Therefore, the PDF of user $n$'s embedding can be written as 
\begin{equation}
\begin{aligned}
f(u_n;\alpha_n, \beta_n) =& f(u_n^i;\alpha_n^i, \beta_n^i)\prod_{j=1,j\not=i}^d f(u_n^j;\alpha_n^j, \beta_n^j)\\
=&c_n^i(1) f(u_n^i;\alpha_n^i+1, \beta_n^i)\prod_{j=1,j\not=i}^d f(u_n^j;\alpha_n^j, \beta_n^j)\\
&+ c_n^i(2) f(u_n^i;\alpha_n^i, \beta_n^i+1)\prod_{j=1,j\not=i}^d f(u_n^j;\alpha_n^j, \beta_n^j)\\
=& c_n^i(1) f(u_n;\alpha'_n, \beta_n) + c_n^i(2) f(u_n;\alpha_n, \beta'_n),
\end{aligned}
\end{equation}
where $(\alpha'_n)^j=\alpha_n^j(j\not= i)$, $(\alpha'_n)^i=\alpha_n^i+1$, and $(\beta'_n)^j=\beta_n^j(j\not= i)$, $(\beta'_n)^i=\beta_n^i+1$.

Therefore, we provide another mixture of user embeddings as follows:
\begin{equation}
\begin{aligned}
&\sum_{m=1,m\not=n}^N w_m f(u_n;\alpha_m, \beta_m)\\
&+ w_n c_n^i(1) f(u_n;\alpha'_n, \beta_n) + w_n c_n^i(2) f(u;\alpha_n, \beta'_n),\\
&with\ w_n c_n^i(1)+w_n c_n^i(2)+\sum_{m=1,m\not=n}^N w_m = \sum_{m=1}^N w_m = 1,
\end{aligned}
\end{equation}
which also holds for the CDF of the user embedding distributions, and implies the non-identifiability of the mixture according to Condition~\ref{not_id}. In addition, it is easy to obtain infinitely many different representations by repeatedly applying equation~\ref{decomp} to different dimensions of embeddings from different users.
\end{proof}

\subsection{Proof of Theorem~\ref{gaussian_id}}\label{app:proof:theorem2}
\begin{proof}[Proof Sketch of Theorem~\ref{gaussian_id}]
We first prove that the $l_2$-norm of $\mu_n$ and $\mu_k$ ($k\not= n$) in the following lemma.

\begin{lemma}\label{lemma:bounded_diff}
Under Assumptions~\ref{assumption:iteration} and \ref{assumption:g_norm}, we have
\begin{equation*}
    \parallel \mu_n - \mu_k \parallel^2 \le 4\eta T^2 G^2. 
\end{equation*}
\end{lemma}

Lemma~\ref{lemma:bounded_diff} indicates that the distance between $\mu_n$ and $\mu_k$ are within $2\parallel\eta TG\parallel$. 

We next show the probability of $\Pr(u_n|\mathcal{U}_n)\le \Pr(u_n|\mathcal{U}_k)$ given the distance between $\mu_n$ and $\mu_k$ in the following lemma.

\begin{lemma}\label{lemma:pr_distance}
For a sample $u$ sampled from $\mathcal{U}_n=\mathcal{N}(\mu_n,\sigma^2 \mathbf{I})$, the probability of $u_n$ not being closer to $\mu_k(k\not= n)$ is as follows:
\begin{equation*}
   \Pr(\parallel u_n-\mu_n\parallel^2 \le \parallel u_n-\mu_k\parallel^2) = \Phi\left(\frac{\parallel \mu_n-\mu_k \parallel}{2\sigma}\right),
\end{equation*}
where $\Phi(\cdot)$ is the CDF of standard gaussian.
\end{lemma}

As the covariance of $\mathcal{U}_k$ is the same as that of $\mathcal{U}_n$ in the proposed design, $\parallel u_n-\mu_n\parallel^2 \le \parallel u_n-\mu_k\parallel^2$ is a sufficient and necessary condition of $\Pr(u_n|\mathcal{U}_n)\le \Pr(u_n|\mathcal{U}_k)$. Substituting Lemma~\ref{lemma:bounded_diff} into Lemma~\ref{lemma:pr_distance}, we have 
\begin{equation}
\Pr(\parallel u_n-\mu_n\parallel^2 \le \parallel u_n-\mu_k\parallel^2) \le \Phi\left(\frac{\parallel\eta TG\parallel}{\sigma}\right)
\end{equation}

Then, under Assumption~\ref{assumption:id}, we have
\begin{equation}
\begin{aligned}
&\Pr(\forall k\not= n,\ \parallel u_n-\mu_n\parallel^2 \le \parallel u_n-\mu_k\parallel^2) \\
= &\prod_{k=1,k\not= n}^N \Pr(\parallel u_n-\mu_n\parallel^2 \le \parallel u_n-\mu_k\parallel^2) \le \left(\Phi\left(\frac{\parallel\eta TG\parallel}{\sigma}\right)\right)^{N-1}
\end{aligned}
\end{equation}

Let $\epsilon$ denote $\left(\Phi\left(\frac{\parallel\eta TG\parallel}{\sigma}\right)\right)^{N-1}$, we finally have 
\begin{equation}
    \Pr(\arg\max_k\Pr(\mathcal{U}_k|u_n) \not= n)\ge 1-\epsilon.
\end{equation}

\end{proof}

\subsection{Proof of Lemma}\label{app:proof:lemmas}

\begin{proof}[Proof of Lemma~\ref{lemma:id_condition}]
    Please refer to the main Theorem in \cite{finitemixture_id_2}.
\end{proof}

\begin{figure*}[!t]
\centering
    \centering
    \subfigure[T5, Gaussian distribution]{
    \includegraphics[height=0.13\textwidth]{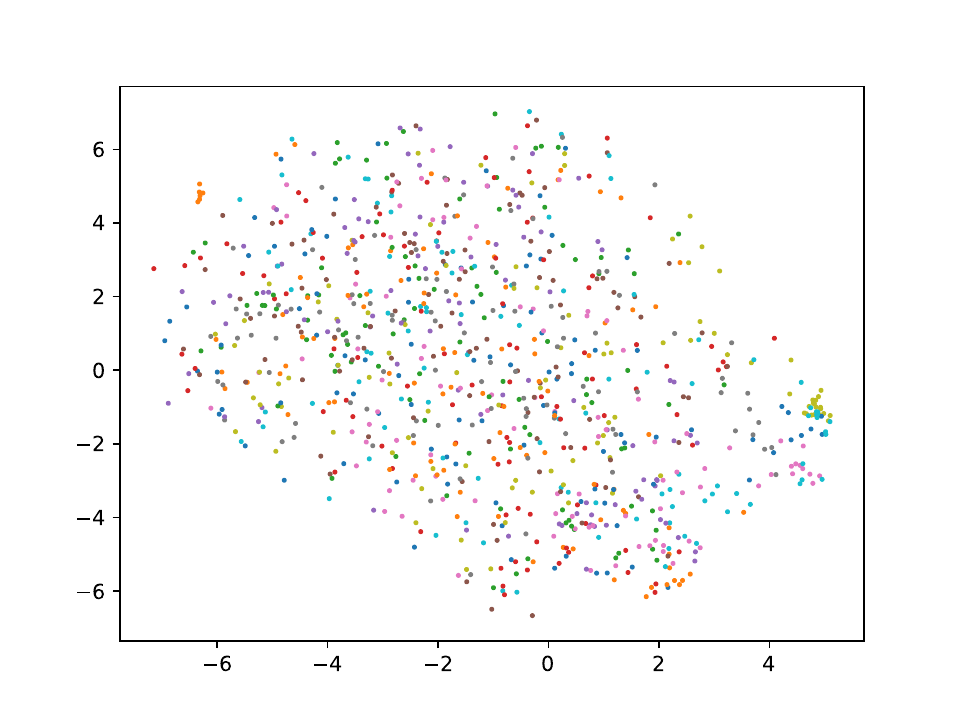}
    }
    \subfigure[T5, Beta distribution]{
    \includegraphics[height=0.13\textwidth]{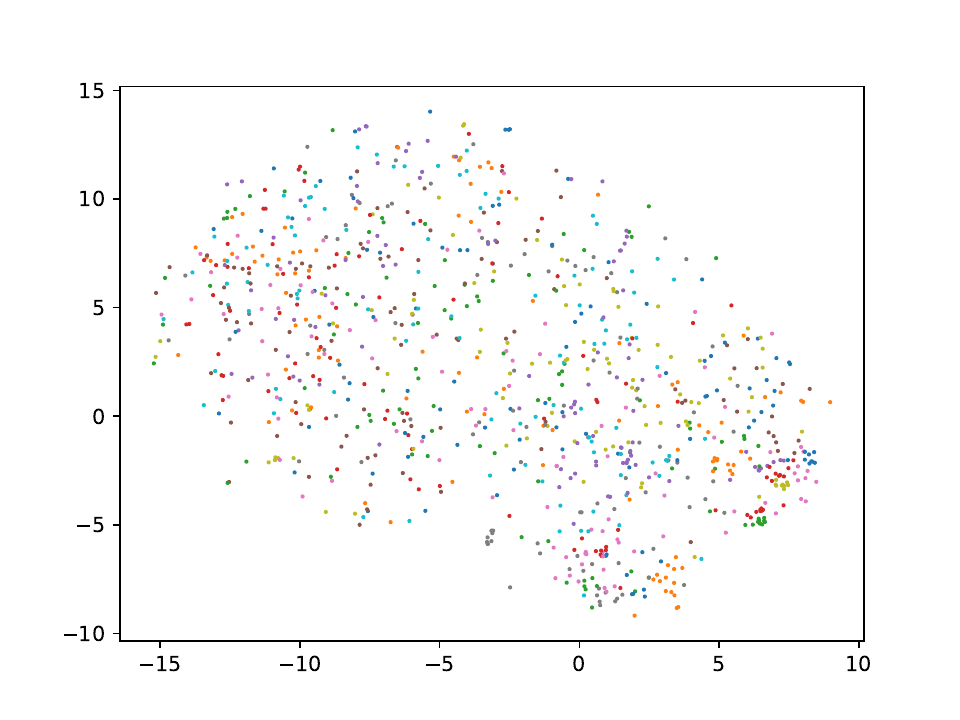}
    }
    \subfigure[BART, Gaussian distribution]{
    \includegraphics[height=0.13\textwidth]{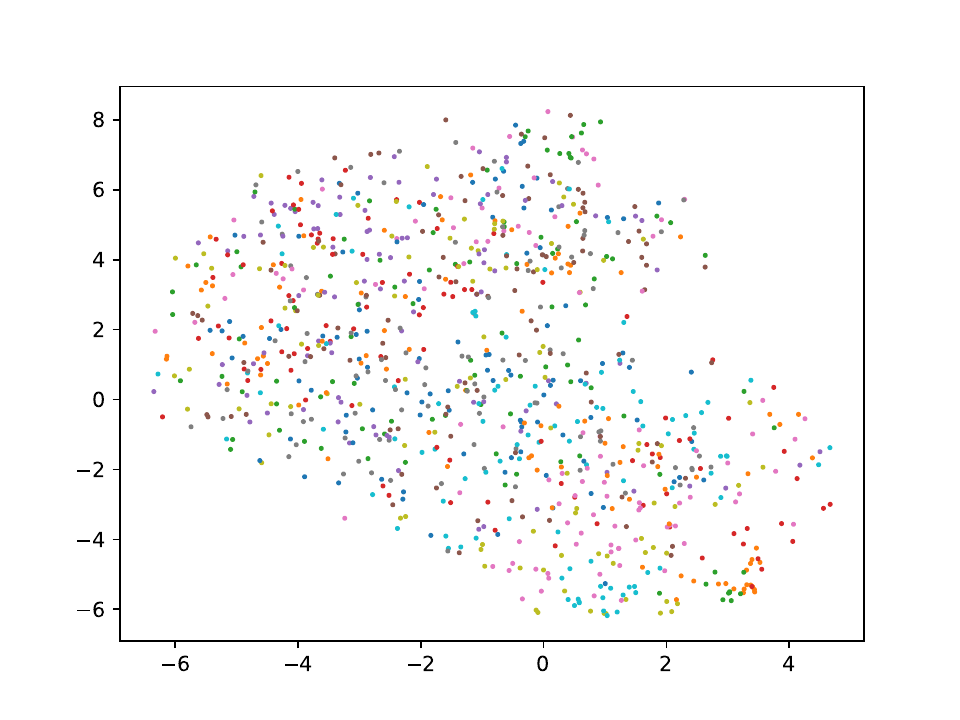}
    }
    \subfigure[BART, Beta distribution]{
    \includegraphics[height=0.13\textwidth]{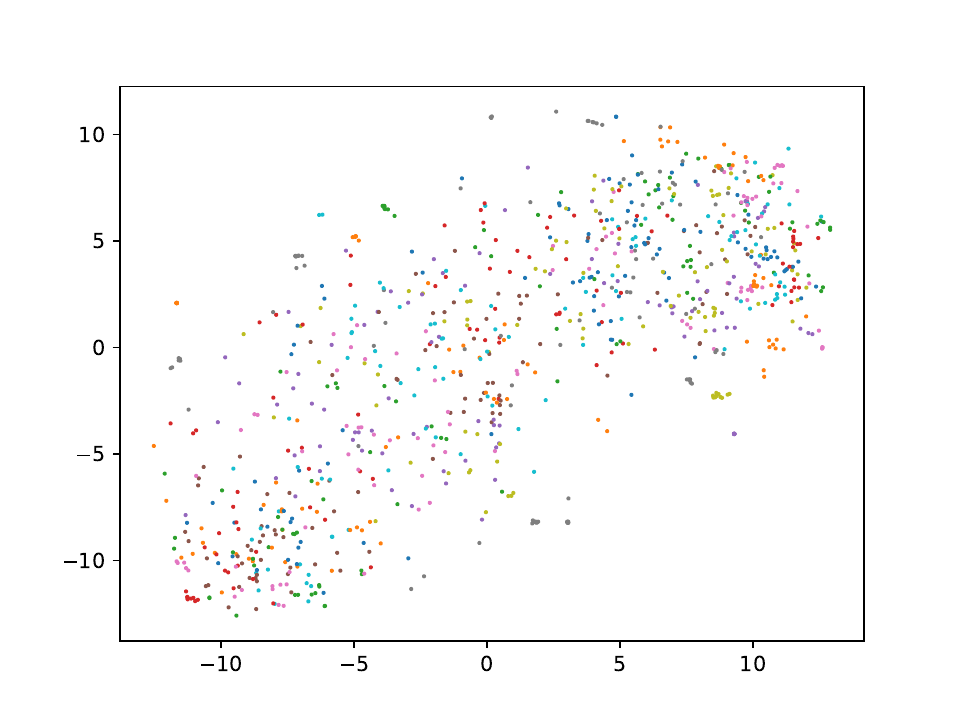}
    }
    \subfigure[GPT2, Beta distribution]{
    \includegraphics[height=0.13\textwidth]{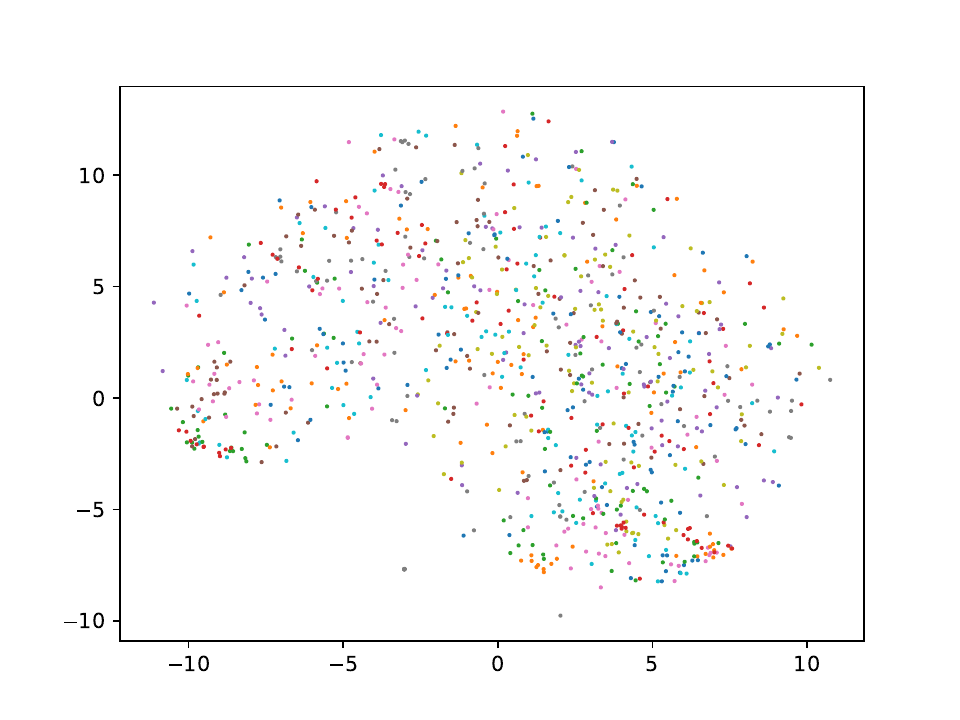}
    }
    \caption{T-SNE visualization of user embeddings under Gaussian and Beta distributions with different models over Amazon-Kindle dataset.}
    \label{fig:app:tsne}
\end{figure*}

\begin{proof}[Proof of Lemma~\ref{lemma:bounded_diff}] Let ${\mu}_{*}$ denote the initial mean of all users, $T_k$ denotes the number of iterations during user $k$'s local training, we then have
\begin{equation}
\begin{aligned}
\parallel \mu_n - \mu_k \parallel^2 &= \parallel (\mu_* - \sum_{t=1}^{T_k}\eta G(\hat{\nabla}_{\mu_k})) - (\mu_* - \sum_{t=1}^{T_n}\eta G(\hat{\nabla}_{\mu_n})) \parallel^2 \\
& = \parallel \sum_{t=1}^{T_k}\eta G(\hat{\nabla}_{\mu_k}) - \sum_{t=1}^{T_n}\eta G(\hat{\nabla}_{\mu_n}) \parallel^2 \\
& \le 2\parallel \sum_{t=1}^{T_k}\eta G(\hat{\nabla}_{\mu_k}) \parallel^2 + 2\parallel \sum_{t=1}^{T_n}\eta G(\hat{\nabla}_{\mu_n}) \parallel^2 \\
& \le 2\eta^2 T_k \sum_{t=1}^{T_k} \parallel G(\hat{\nabla}_{\mu_k}) \parallel^2 + 2\eta^2 T_n \sum_{t=1}^{T_n} \parallel G(\hat{\nabla}_{\mu_n}) \parallel^2 \\
& \overset{(a)}{\le} 4\eta^2 T^2 G^2,
\end{aligned}
\end{equation}
where (a) follows from Assumptions~\ref{assumption:iteration} and \ref{assumption:g_norm}.
\end{proof}

\begin{proof}[Proof of Lemma~\ref{lemma:pr_distance}]
First, translate and rotate the coordinate system to make $\mu_n$ be the origin and the first axis align with $\mu_k-\mu_n$. Then, let $\{x_1,x_2,\cdots,x_d\}$ denotes the new coordinate of $u$, when $\parallel u-\mu_n \parallel\le \parallel u-\mu_k \parallel$, we have
\begin{equation}
    x_1^2 + \sum_{i=1}^d x_i^2 \le (x_1-\parallel \mu_k-\mu_n \parallel)^2 + \sum_{i=1}^d x_i^2,
\end{equation}
which indicates $x_1 \le \frac{\parallel \mu_k-\mu_n \parallel}{2}$. By the isotropic of $\mathcal{U}_n$, we then have $x_1\sim \mathcal{N}(0,\sigma^2)$, so that
\begin{equation}
\begin{aligned}
    \Pr(x_1 \le \frac{\parallel \mu_k-\mu_n \parallel}{2}) &= \Pr(\frac{x_1}{\sigma}\le \frac{\parallel \mu_k-\mu_n \parallel}{2\sigma})\\
    &= \Phi(\frac{\parallel \mu_k-\mu_n \parallel}{2 \sigma}),
\end{aligned}
\end{equation}
where $\Phi(\cdot)$ denotes the CDF of standard gaussian.
\end{proof}




\section{Visualization of User Embeddings}\label{app:tsne}

We demonstrate additional results of T-SNE visualization of user embeddings in Figure~\ref{fig:app:tsne}. The user embeddings are sampled from Gaussian or Beta distribution, trained with GPT2, T5 or BART over Amazon-Kindle Review dataset. Under experiment settings, most of the embeddings from different users are mixed together, preventing the cloud from identifying data from different users

\end{document}